%% file: arxiv_2026.tex
\documentclass{article}
\usepackage{stmaryrd}

\usepackage{arxiv}

\usepackage{doi}

\usepackage[utf8]{inputenc} 
\usepackage[T1]{fontenc}    
\usepackage{hyperref}       
\usepackage{url}            
\usepackage{booktabs}       
\usepackage{amsfonts}       
\usepackage{nicefrac}       
\usepackage{microtype}      
\usepackage{xcolor}         
\usepackage{amsmath, amsthm}
\usepackage{stmaryrd}
\usepackage{enumitem}
\usepackage{color} 
\usepackage{soul}
\usepackage{bbm}
\usepackage{amssymb}
\usepackage{graphicx}
\usepackage{natbib}
\usepackage[textsize=tiny, color=green!10!white]{todonotes}
\usepackage{bm}
\usepackage{thmtools}

\newtheorem{lemma}{Lemma}

\usepackage{etoolbox}

\title{Universal Transformers for Circuit Computations:
Perfect Length Generalization in Tiny Transformers}

\renewcommand{\headeright}{}
\renewcommand{\undertitle}{}
\renewcommand{\shorttitle}{Universal Transformers for Circuit Computations}

\author{%
  Takuya Ito\thanks{takuya.ito@ibm.com}\\
  IBM Research\\
  \And
  Ruchir Puri \\
  IBM Research\\
  \And
  Murray Campbell \\
  IBM Research\\
  \And
  Parikshit Ram \\
  IBM Research\\
}

\begin{document}

\maketitle

\begin{abstract}
Learning generalizable algorithmic computations remains a challenge for neural networks, as reflected in persistent failures on compositional and length generalization benchmarks. 
We present a provably correct, transformer parameterization (with only 280 learnable parameters for Boolean algebra tasks) capable of learning and evaluating problems of {\em any depth or length}. 
We assume inputs are fully parenthesized, well-formed expressions.
Our approach conceptualizes algorithmic tasks as circuit models embedded in transformers, enabling depth-1 circuit reduction in a single forward pass.
To achieve depth generalization, we introduce a positional encoding that tracks each gate's depth within the circuit, enabling the model to identify evaluable subexpressions at each iteration via masked hard attention, with $O(n)$ per-iteration complexity via linear attention.
Combined with an autonomous halting criterion, the model terminates after $d$ iterations for problems of depth $d$, yielding $O(n \cdot d)$ total complexity.
We show that training on shallow problem instances (depth 1 and depth 2) effectively recovers interpretable parameters that {\em snap} into place, resulting in exact length generalization. 
Though we establish that our construction provably evaluates Boolean expressions -- a universal symbolic computation -- of arbitrary length perfectly, in other experiments we also demonstrate that our transformer variant can learn and generalize perfectly (100\% accuracy) on other common length generalization benchmarks, including modular arithmetic and ListOps.

\end{abstract}


\section{Introduction}
\label{sec:introduction}

Understanding how transformers can execute exact symbolic computations remains a challenge in artificial intelligence.
Standard transformer architectures have struggled with symbolic generalization \citep{lake_generalization_2018}, particularly when generalizing to longer input sequences \citep{hupkes_compositionality_2020,dziri_faith_2023,zhou_what_2024}.
Although recent works have reported substantial progress on algorithmic benchmarks -- such as arithmetic reasoning and formal language recognition -- and in some cases have identified interpretable internal representations \citep{he_learning_2024,cho_arithmetic_2025,fan_looped_2025,adler_towards_2025}, these advances often rely on task-specific heuristics and lack guarantees of exact algorithmic correctness outside the training distribution. 
Thus, strong benchmark performance does not ensure that a transformer learns a generalizable algorithm.

We first address this by studying transformers on Boolean formulas -- a canonical family of compositional computations that can be represented as circuits \citep{arora_computational_2009}.
Circuits provide explicit algorithmic structure -- gates, edges, and intermediate values -- and naturally scale to arbitrary sizes (Fig. \ref{fig:problem}).
Many string-based tasks (e.g., modular arithmetic and context-free grammars) can be mapped to circuits, yielding a unified notion of algorithmic generalization via a graph-based structure \citep{ito_quantifying_2025}.
Here, we parameterize by hand a fixed, small transformer (280 parameters) that provably evaluates arbitrary non-uniform Boolean expressions via a circuit reduction procedure.
When randomly initialized, the model can learn to generalize perfectly from depth-1 and depth-2 samples.
This behavior is enabled by a novel positional encoding (PE) mechanism that infers circuit depth from a parenthesized string and identifies evaluable subexpressions at each step.
The model is also computationally efficient: attention scales linearly in memory, and the number of iterations scales approximately logarithmically in input length, suggesting that exact algorithmic computation in transformers can be realized without sacrificing efficiency or parameter compactness.

\begin{figure}[h]
  \centering
  \includegraphics[scale=1.1]{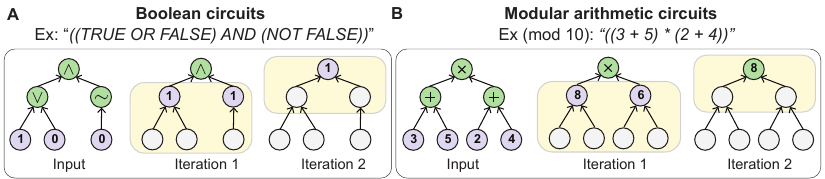}
  \caption{
  Example A) Boolean and B) modular arithmetic expressions, encoded as circuits. 
  Circuits provide a structured algorithm to evaluate an expression by following the edges from the input gates to the output gate.
  We illustrate circuit evaluation through circuit reduction, whereby each operator gate is iteratively updated by the subexpression's computed value until the output is obtained.
  }
  \label{fig:problem}
\end{figure}

To demonstrate generality beyond Boolean formulas, we show that the same architecture learns and perfectly generalizes to other common compositional tasks, such as modular arithmetic and ListOps \citep{nangia_listops_2018} from shallow training data.
Together, these results provide an example of a transformer that provably implements an exact circuit algorithm while remaining interpretable, efficient, and learnable.
We present our findings first by introducing the exact model construction that can evaluate Boolean formulas without any training, followed by empirical learnability experiments for all tasks in randomly initialized models.

{\bf Contributions.}
\begin{enumerate}[leftmargin=2em,nosep]
    \item A provably-correct, white-box transformer construction for non-uniform (arbitrary depth) Boolean expressions, enabling perfect evaluation via a circuit reduction procedure.
    \item Empirical results (and code) demonstrating that this architecture can learn and perfectly generalize common compositional tasks, including Boolean algebra, modular arithmetic, and ListOps.
\end{enumerate}

\section{Preliminaries}
\label{sec:problem}

We begin by introducing notation and problem preliminaries.
For clarity, we focus on Boolean expressions, which provide the simplest setting for encoding universal symbolic computation and requires only an embedding dimension of 7 to one-hot encode its entire vocabulary.
While we focus on this language for exposition, the general model construction applies to any domain that can be represented as a circuit, with only superficial changes to the model (e.g., changes to the requisite embedding dimension according to the size of the task vocabulary). 

We use a \textbf{universal Boolean language}: $\Sigma_{\mathrm{bool}} = \big(\texttt{(}, \texttt{)}, \texttt{1}, \texttt{0}, \wedge, \vee, \sim \big)$.
Operands are \texttt{1} (\texttt{TRUE}) and \texttt{0} (\texttt{FALSE}). 
Operators are unary $\sim$ (\texttt{NOT}) and binary $\wedge$ (\texttt{AND}), $\vee$ (\texttt{OR}).
See Fig. \ref{fig:problem}A for examples of expressions, and Appendix \ref{app:task} for further details.
\textbf{Token Embedding.}
Let $D = |\Sigma_{\mathrm{bool}}|$. 
Each token is encoded as a one-hot vector in $\{0,1\}^D$. 
Note that $\Sigma_{\mathrm{bool}}$ is an ordered vocabulary. 
Thus, the embedding space for each token of $\Sigma_{\mathrm{bool}}$ is mapped to the embedding space $\{0,1\}^{D}$, where the $i$-th coordinate corresponds to the $i$-th token in $\Sigma$.
An input expression of length $n$ is represented by

\begin{equation}
    \label{eq:encoding}
    X^{(0)} = (x^{(0)}_1,\dots,x^{(0)}_n) \in \{0,1\}^{n \times D}
\end{equation}
We write $X^{(t)}$ for the sequence state after $t$ model iterations.

\paragraph{Terminology.}
Let $\llbracket x \rrbracket$ denote the \textbf{\emph{semantic value}} (or output) of a Boolean expression $x$ (i.e., $\texttt{eval(x)}$), such that $\llbracket x \rrbracket \in \{\texttt{1},\texttt{0}\}$. 
Let $d(x)$ denote the \textbf{\emph{depth}} of $x$ represented as a circuit, and \textbf{\emph{active depth}} $\delta(X)$ as the maximum depth among unreduced subexpressions in $X$.
For fully parenthesized expressions, depth is defined recursively: a (\texttt{0} or \texttt{1}) has depth $0$, while a compound expression has depth one greater than the deepest of its constituent subexpressions.
For fully parenthesized expressions, a \textbf{\emph{reducible subexpression}} is an innermost parenthesized span containing only an operator and its operands.
A sequence state $X^{(t)} \in \{0,1\}^{n \times D}$ is a \textbf{\emph{valid expression}} if each innermost subexpression corresponds either to an unevaluated reducible subexpression or its semantic value, and its \textbf{\emph{global evaluation}} (final semantic value $\llbracket X^{(t)} \rrbracket$ obtained by recursively evaluating all subexpressions) is well-defined.
Two valid expressions $X$ and $Y$ are \textbf{\emph{semantically equivalent}}, written $X \equiv Y$, if they have the same global evaluation.

{\bf Model Overview.}
The model evaluates expressions by iterative local circuit reduction. At each iteration it reduces the expression's depth by 1 through the following steps:

\begin{enumerate}[leftmargin=2em,nosep]
    \item Computes a reduction mask identifying the deepest reducible subexpressions.
    \item Decomposes that mask into parallel subexpressions that can be evaluated independently.
    \item Routes token information {\em within a subexpression} with attention.
    \item Applies a local semantic map to compute the semantic value of each subexpression via MLP.
    \item Evaluates masked subexpressions while preserving non-reduced tokens with a gated residual.
\end{enumerate}

\section{Model Construction and Formalization}
\label{sec:model}

The following formal argument proceeds step by step by crafting an exact parameterization.
For each model component, we formally construct each component and summarize its key contribution in a corresponding lemma.
Throughout, we also provide figures that illustrate, with concrete examples, how these components collectively implement circuit depth reduction (Fig. \ref{fig:pexattn}-\ref{fig:thm}).
We combine these ingredients into theorem statements for depth-1 circuit reduction and arbitrary-depth correctness.
By construction, the model does not require training and achieves perfect performance, correctly evaluating Boolean expressions of arbitrary depth.
(In the subsequent section, we explore learnability under random initialization given the base architecture.)
For our formal claims, we make the following assumptions:
1) Inputs are fully parenthesized (balanced parentheses) valid expressions.
2) Tokens within a valid subexpression are order-invariant (commutative up to parentheses), e.g., (\texttt{1} $\vee$ \texttt{0}) $\equiv$ (\texttt{0} $\vee$ \texttt{1}).
We visualize all the chosen model parameters in Appendix Fig. \ref{afig:fullmodel}, and provide the code for this solution in the supplementary zip file.



\subsection{Positional encoding: Identifying a reduction mask and parallel subexpressions}

The key insight that enables a single-layer transformer to implement depth-1 circuit reduction is that representing a circuit from a string fundamentally requires encoding the circuit's depth (i.e., by counting parentheses).
In what follows, we describe a novel PE mechanism that, given a valid expression $X^{(t)}$, is able to identify the deepest reducible subexpression(s).

\paragraph{Reduction mask.} We construct a PE mechanism with a gating matrix parameter $G_{\text{PE}} \in \{-1,0,1\}^{D \times D}$ that produces a \emph{reduction mask} $P^{(t)} \in \{0,1\}^n$ identifying the tokens belonging to the deepest unreduced subexpressions, i.e., those at depth $\delta(X^{(t)})$.
Formally, $P^{(t)}_i = 1$ if and only if the $i$-th token of $X^{(t)}$ belongs to a maximally nested (deepest) parenthesized subexpression (i.e., it is reducible).
Identifying such subexpressions reduces to computing parenthesis nesting depth.
Let "\texttt{(}" and "\texttt{)}" correspond to the first and second embedding dimensions, respectively.
We define
\begin{equation}
\label{eq:gpe}
G_{\mathrm{PE}} = \mathrm{diag}(1,-1,0,0,\dots,0),
\end{equation}
so that when an expression $X^{(t)}$ is gated (multiplied) by $G_\text{PE}$, open parentheses contribute $+1$, closed parentheses contribute $-1$, and all other tokens are ignored.

To count the number of open/closed parentheses, we compute the cumulative sum of parentheses in $X^{(t)}$, both forward and backwards along the sequence.
Let $L \in \{0,1\}^{n \times n}$ be a lower-triangular matrix of ones. 
The model computes the cumulative sum of parenthesis counts by gating $X^{(t)}$ with $G_\text{PE}$, and then computes the cumulative sum forwards and backwards with $L$ to maintain symmetry:
\begin{equation}
\label{eq:pe_a1a2}
A_1^{(t)} = (L X^{(t)} G_{\mathrm{PE}})\mathbf{1}, \qquad
A_2^{(t)} = -(L^\top X^{(t)} G_{\mathrm{PE}})\mathbf{1},
\end{equation}
where $A_1^{(t)}$ counts the net number of open parentheses from position $1$ to $n$, and $A_2^{(t)}$ counts the net number of close parentheses from position $n$ to $1$. 
These are then stacked into
\begin{equation}
\label{eq:pe_astack}
A^{(t)} = \mathrm{stack}(A_1^{(t)},A_2^{(t)}) \in \mathbb{Z}^{n \times 2},
\end{equation}
and a hardmax operation is applied to identify sequence positions with maximal depth:
\begin{equation}
\label{eq:pe_p}
P^{(t)} = \mathbbm{I}[\max_{j \in \{1,2\}} A^{(t)}_{i,j} = \max_{i',j'} A^{(t)}_{i',j'}] \in \{0,1\}^n,
\end{equation}
where $\mathbbm{I}[\cdot]$ is the indicator function.
This selects positions $i$ where the maximum depth value across both forward and backward counts equals the global maximum depth, identifying the deepest nested positions (e.g., see Fig. \ref{fig:pexattn}A for intuition).
Thus, the resulting mask $P^{(t)}$ selects exactly the tokens belonging to the deepest reducible subexpressions.
(Note that in learnable experiments reported in Section \ref{sec:empirical}, $G_{\mathrm{PE}}$ is set to be learnable.)

\begin{lemma}[Reduction mask]
\label{lem:reduction-mask}
Define parameter $G_\text{PE}$ as in Eq. \ref{eq:gpe}.
Let $X^{(t)}$ be a valid expression.
The PE mechanism equipped with $G_{\text{PE}}$ and described in Eqs. \ref{eq:pe_a1a2}-\ref{eq:pe_p} produces the binary reduction mask $P^{(t)}$, whose support consists exactly of the tokens in the deepest parenthesized spans of $X^{(t)}$. 
\end{lemma}

\begin{proof}
    The claim follows directly from the construction in Eqs. \ref{eq:gpe}-\ref{eq:pe_p}.
    A visual illustration with an example can be found in Fig. \ref{fig:pexattn}A.
\end{proof}

\begin{figure}[h]
  \centering
  \includegraphics[scale=1.1]{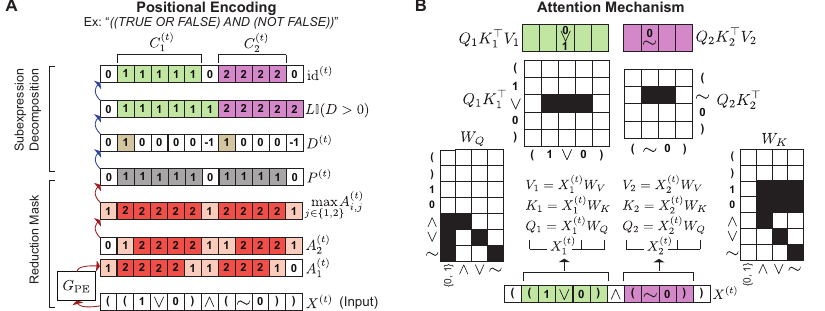}
  \caption{
  A visual demonstration of the PE and attention mechanism constructions (Lemmas \ref{lem:reduction-mask}-\ref{lem:attn}. 
  A) PE gating identifies reducible subexpressions by first computing the reduction mask $P^{(t)}$, and then decomposing the independent reducible subexpressions $C_1^{(t)}$ and $C_2^{(t)}$.
  Beginning with the specific example input ``\texttt{((1}$\vee$\texttt{0)}$\wedge$\texttt{(}$\sim$\texttt{0))}'' ($X^{(t)}$), we characterize and label each of the PE operations and variable assignments.
  B) Type-constrained attention mechanism is implemented on each subexpression cluster (green and pink masks) generated from the PE. 
  $W_Q$ and $W_K$ represent the actual $\{0,1\}^{D \times 4}$ matrix construction our solution (Lemma \ref{lem:attn}) uses.
  By implementing attention on each reducible subexpression independently, we constrain attention routing to independent subexpressions.
}
  \label{fig:pexattn}
\end{figure}

\paragraph{Parallel subexpression decomposition.}
Given the binary reduction mask $P^{(t)}$ from Lemma \ref{lem:reduction-mask} generated from parameter $G_{\text{PE}}$, multiple disjoint subexpressions may be reducible in parallel for expressions $x$ with $d(x)>1$.
Thus, the model must decompose these subexpressions into separate clusters that can be independently evaluated.

We now construct a PE submechanism that partitions $P^{(t)}$ into disjoint contiguous sets $\mathcal{C}^{(t)}=\{C_1^{(t)},\dots,C_{k_t}^{(t)}\}$, such that each cluster $C_c^{(t)}$ corresponds to exactly one locally reducible subexpression, and distinct clusters can be evaluated independently in parallel (Fig. \ref{fig:pexattn}A).
Let $\mathbf{0}_1 \in \{0,1\}^1$ denote a leading zero. 
Define the padded mask $\widetilde{P}^{(t)} = [\mathbf{0}_1; P^{(t)}] \in \{0,1\}^{n+1}$.
The discrete difference operator identifies transitions from unmasked to masked positions:
\begin{equation}
\label{eq:discretediff}
D^{(t)}_i = \widetilde{P}^{(t)}_i - \widetilde{P}^{(t)}_{i-1}, \quad i=2,\dots,n+1.
\end{equation}
Each positive transition ($D^{(t)}_i > 0$) marks the start of a new subexpression. Cumulative summation of cluster starts assigns a unique integer ID to each contiguous masked segment, then unmasked positions are zeroed by multiplying with the original reduction mask $P^{(t)}$:
\begin{equation}
    \label{eq:id_i}
    \mathrm{id}^{(t)} = P^{(t)} \odot L \mathbb{I}(D^{(t)} > 0)
\end{equation}
This assigns each position a cluster ID: $\mathrm{id}^{(t)}_i = 0$ for unmasked positions, and $\mathrm{id}^{(t)}_i = c$ (where $c \in \{1,\dots,k_t\}$) for positions in the $c$-th subexpression cluster. 
The total number of clusters $k_t$ equals the number of $0 \to 1$ transitions in $P^{(t)}$.
Each cluster corresponds to a contiguous interval of positions. For cluster ID $c \in \{1,\dots,k_t\}$, define the position set
\begin{equation}
\label{eq:pe_clust}
C_c^{(t)} = \{i \in \{1,\dots,n\} : \mathrm{id}^{(t)}_i = c\},
\end{equation}
which forms a contiguous interval $C_c^{(t)} = \{i_c, i_c+1, \dots, j_c\}$ for $1 \le i_c \le j_c \le n$. The collection of all clusters is thus
\begin{equation}
\label{eq:pe_clustcollection}
\mathcal{C}^{(t)} = \{C_1^{(t)}, C_2^{(t)}, \dots, C_{k_t}^{(t)}\}.
\end{equation}

Because the support of $P^{(t)}$ consists of disjoint contiguous reducible spans, each cluster $C_c^{(t)}$ coincides with one span (subexpression) such that each subexpression can be evaluated independently.

\begin{lemma}[Parallel subexpression decomposition]
\label{lem:cluster-correctness}
Let $P^{(t)}$ be the binary reduction mask specified in Lemma \ref{lem:reduction-mask}, whose support is the disjoint union of reducible subexpressions, such as the example provided in Fig. \ref{fig:pexattn}A). Then the clustering vector $\mathrm{id}^{(t)}$ in Eq. \ref{eq:id_i} partitions the support of $P^{(t)}$ into disjoint contiguous clusters $\mathcal{C}^{(t)}=\{C_1^{(t)},\dots,C_{k_t}^{(t)}\}$ (Eq. \ref{eq:pe_clustcollection})
such that each cluster $C_c^{(t)}$ corresponds to exactly one locally reducible subexpression, and distinct clusters can be evaluated in parallel.
\end{lemma}

\begin{proof}
    The claim follows directly from the construction in Eqs. \ref{eq:discretediff}-\ref{eq:pe_clustcollection}.
    A visual illustration with an example can be found in the upper half of Fig. \ref{fig:pexattn}A.
\end{proof}

\subsection{Attention Mechanism: Type-constrained token routing}

By Lemma \ref{lem:cluster-correctness}, we constructed disjoint contiguous clusters $\mathcal{C}^{(t)}$ of $X^{(t)}$ that can be evaluated in parallel.
Next, we apply attention separately to each subexpression $C_c^{(t)}$. 
This ensures that attention is limited to a single, reducible span (subexpression).
Moreover, reducible subexpressions encode a local, depth-1 circuit structure.
In Boolean algebra, this local structure follows a grammar that obeys strict type constraints: operands map to operators (e.g., in ``\texttt{(1$\vee$0)}'', \texttt{1} $\to \vee$ and \texttt{0} $\to \vee$; Fig. \ref{fig:problem}A). 
We leverage these constraints to construct token type-specific mappings within the attention mechanism, enabling faithful representation of Boolean grammars.

\emph{Local and type-based attention}. 
Without loss of generality, choose $C_c^{(t)} \in \mathcal{C}^{(t)}$, and let
\begin{equation}
\label{eq:x_c}
X_c^{(t)} = \mathbbm{I}[i \in C_c^{(t)}] \odot X^{(t)}
\end{equation}
 denote the token embeddings for subexpression $C_c^{(t)}$. Let queries, keys, and values be given by 
\begin{equation}
\label{eq:querykeyval_def}
Q_c^{(t)} = X_c^{(t)} W_Q, \quad K_c^{(t)} = X_c^{(t)} W_K, \quad V_c^{(t)} = X_c^{(t)} W_V,
\end{equation}
where $W_Q, W_K \in \{0,1\}^{D \times d_h}$ are binary routing matrices that permit only binarized admissible interactions between token types (operands to operators, and operator self-interactions), and 0 everywhere else; see Fig. \ref{fig:pexattn}B (or Fig. \ref{afig:fullmodel}) for the exact matrix constructions for $W_Q$ and $W_K$.
Specifically for Boolean expressions, $D=7$ and $d_h=4$, where $d_h$ encodes one dimension for all operands (\texttt{0} and \texttt{1}), and a dimension each for each operator separately ($\wedge, \vee$ and $\sim$).
$W_V=I_D$ is fixed as the identity.
Notably, the product $Q_c^{(t)} \cdot (K_c^{(t)})^\top$ produces an attention weight matrix that routes the operand and operator embeddings within $C_c^{(t)}$ to the operator's token position (Fig \ref{fig:pexattn}B).

\emph{Bag-of-tokens aggregation.}
Within $C_c^{(t)}$, $V_c^{(t)}=X_c^{(t)}$, since $W_V=I_D$.
Thus, the attention output aggregates token embeddings via
\begin{equation}
\label{eq:attn_qkv}
H_c^{(t)} = Q_c^{(t)} (K_c^{(t)})^\top X_c^{(t)}
\end{equation}
Critically, contributions are summed over tokens locally in $C_c^{(t)}$ (and therefore do not interact across subexpressions) and depend only on their routed types.
By construction of $Q_c^{(t)}$ and $K_c^{(t)}$ (Fig. \ref{fig:pexattn}B), the resulting vector is a ``bag-of-words embedding that simply sums the admissible token embeddings (operands with operators) at the operator's token position using only the admissible tokens in $C_c^{(t)}$, and leaves all other positions in $C_c^{(t)}$ as 0 vectors.
Note that the attention mechanism does not use a softmax.
Thus, the operation in Eq. \ref{eq:attn_qkv} reduces to linear attention, with $Q_c^{(t)} \big((K_c^{(t)})^\top X_c^{(t)}\big)$, which allows us to avoid the quadratic cost of standard softmax attention \citep{katharopoulos_transformers_2020}. 

\begin{lemma}[Type-constrained attention mechanism]
\label{lem:attn}
Using $W_Q$ and $W_K$ as defined in Eq. \ref{eq:querykeyval_def} and exactly constructed in Fig. \ref{fig:pexattn}B), the attention output at the operator position computes a permutation-invariant summed aggregation of tokens in $C_c^{(t)}$, yielding a bag-of-tokens (counts) vector of the admissible token types in the reducible subexpression at the token embedding location of the operator. 
\end{lemma}
\begin{proof}
    The claim follows directly from the construction in Eqs. \ref{eq:x_c}-\ref{eq:attn_qkv}.
    A visual illustration with an example can be found in Fig. \ref{fig:pexattn}B.
\end{proof}

\subsection{Feedforward Layer: Local semantic evaluation}

The output of the preceding Lemma \ref{lem:attn} is a bag-of-words vector of a locally reducible subexpression.
For reducible Boolean expressions, the semantic value can be fully determined by the operator type together with the pair of operand values, independent of order.
This mapping can be realized as a finite look-up table of depth-1 Boolean tree.
Since the set of these valid reducible subexpressions is finite, $f$ is a function on a finite domain. 
Any function on a finite domain is realizable by a single-hidden-layer MLP \citep{hornik_multilayer_1989}.
In practice, in our constructed model solution, we train a single-hidden-layer MLP to simulate a finite truth table for depth-1 Boolean expressions and append this module to our transformer, although a fixed dictionary look-up can also be straightforwardly implemented (see Code and Appendix Fig. \ref{afig:fullmodel}A).

\begin{restatable}[Feedforward lookup table]{lemma}{mlplemma}
\label{lem:mlp}
There exists a function $f: \mathbb{Z}^V \to \{0,1\}^D$ such that, when applied to the bag-of-word token representation of any valid reducible subexpression $C_c^{(t)}$, it outputs the correct one-hot embedding of the resulting subexpression, and a 0 vector otherwise.
\end{restatable}

\begin{proof}
   The proof follows by a standard finite-domain argument and is deferred to Appendix~\ref{app:proof_mlp}.
\end{proof}


\subsection{Gated residual update}

After the feedforward network (Lemma \ref{lem:mlp}), reducible subexpressions produce a single token embedding that correctly evaluates that subexpression at the token position of that subexpression's operator, and 0 vectors everywhere else (see Fig. \ref{fig:thm}).
However, to evaluate the full circuit beyond depth-1 expressions, we require a mechanism to preserve unreduced tokens/expressions from $X^{(t)}$ to $X^{(t+1)}$.
Concretely, once a subexpression is reduced, its interior tokens (the operands and the consumed parentheses) become the $\mathbf{0}$ vector, and only the operator position carries the computed value forward. The gated residual then leaves these zeroed positions untouched at subsequent iterations.

\begin{restatable}[Gated residual connection]{lemma}{gatedlemma}
\label{lem:residual-preservation}
Let $P^{(t)} \in \{0,1\}^n$ be the reduction mask specified in Lemma \ref{lem:reduction-mask} and $M^{(t)} = \mathbf{1} - P^{(t)}$ its complement. Define the gated residual update
\begin{equation}
    \label{eq:gated_resid}
    X^{(t+1)} = M^{(t)} \odot X^{(t)} + P^{(t)} \odot Y^{(t)}.
\end{equation}
Then every position outside the reduction mask is preserved, and only positions inside the reduction mask are updated.
\end{restatable}

\begin{proof}
    Deferred to Appendix \ref{app:proofs}
\end{proof}


\subsection{Depth-1 circuit reduction and global circuit evaluation}

The preceding lemmas construct specific model components that collectively imply that a single model iteration implements a depth-1 circuit reduction step (Fig. \ref{fig:thm}).

\begin{restatable}[Depth-1 circuit reduction]{theorem}{thmone}
\label{thm:one-step-reduction}
Let $F$ be one transformer iteration (PE, attention, MLP, and gated residual; Lemmas \ref{lem:reduction-mask}--\ref{lem:residual-preservation}). 
For any valid $X^{(t)}$ with active depth $\delta(X^{(t)})>0$, we construct $F$ such that $X^{(t+1)} = F(X^{(t)})$
simultaneously evaluates every reducible subexpression, replaces each by its correct value, and preserves all inactive positions. 
Then,
\[
X^{(t+1)} \equiv X^{(t)}
\qquad \text{and} \qquad
\delta(X^{(t+1)}) = \delta(X^{(t)}) - 1.
\]
\end{restatable}

\begin{proof}
    The proof follows by combining Lemmas \ref{lem:reduction-mask}-\ref{lem:residual-preservation}, and is deferred to Appendix \ref{app:proofs}.
\end{proof}

\begin{restatable}[Global circuit evaluation]{theorem}{thmtwo}
\label{thm:global-correctness}
For any valid Boolean expression $x$ of depth $d(x)$, with $X^{(t+1)} = F(X^{(t)})$ and $X^{(0)}$ as its embedding,
\[
X^{(d(x))} \equiv \llbracket x \rrbracket \quad \text{and} \quad \delta(X^{(d(x)}) = 0
\]
\end{restatable}

\begin{proof}
    The proof follows by applying Theorem~\ref{thm:one-step-reduction} inductively, and is deferred to Appendix \ref{app:proofs}.
\end{proof}


Importantly, this construction halts autonomously, as $\delta(X^{(t)})$ is computed within the PE mechanism (by Eq. \ref{eq:pe_astack}, when $\sum_{i,j}A_{ij}^{(t)}=0$ and there remain no parentheses in the sequence).
This transformer guarantees exact circuit evaluation, and we visualize the parameters for this model in Fig. \ref{afig:fullmodel}.

\begin{figure}[h]
  \centering
  \includegraphics[scale=1.0]{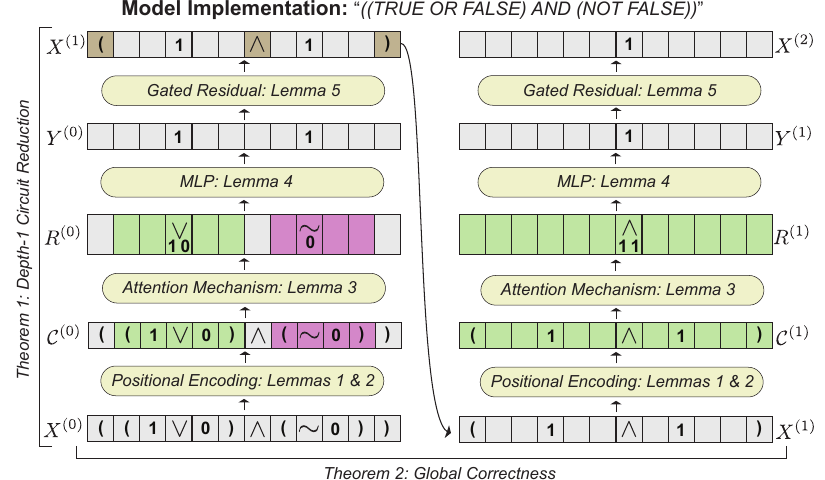}
  \caption{
  An illustrative example of Theorem \ref{thm:one-step-reduction} and \ref{thm:global-correctness}.
  Note that each token position is a $D$-dimensional embedding that encodes a one-hot or counts vector.
  Empty gray token boxes in the sequence denote the 0 vector.
  We visualize evaluation on the depth 2 expression ``\texttt{((1$\vee$0)$\wedge$($\sim$0))}''.
  }
  \label{fig:thm}
\end{figure}


\paragraph{Computational complexity}
Each transformer iteration has near-linear attention cost $O(n)$, with the dominant typed-routing step scaling as $O(kn)$ for $k$ reducible subexpressions, but effectively $O(n)$ for structured inputs since $k \ll n$. 
The number of model iterations is $d(x)$, yielding fewer iterations than standard token-by-token evaluation (e.g., in RNNs), which typically requires $\Theta(n)$ steps \citep{deletang_neural_2022} (Appendix \ref{app:complexity} for further details).
Thus, the total model complexity is $O(n \cdot d(x))$.
Because each additional level of nesting introduces at least one pair of parentheses and one operator, $d(x) \le n/3$.
This implies the cost is always sub-quadratic; the worst case $O(n^2)$ arises only for maximally unbalanced circuits with $d(x) = \Theta(n)$. 
Empirically, input length grows exponentially with depth ($d \approx \log n$), giving $d(x) = O(\log n)$ and an effective cost of $O(n \log n)$ (see Fig. \ref{fig:learn}C).

\section{Numerical Experiments}
\label{sec:empirical}

In this section, we demonstrate that this architecture learns compositional tasks and achieves perfect length generalization, extrapolating to arbitrary depths after training only on shallow instances.

\paragraph{Tasks.}
We evaluated three families of algorithmic tasks: modular arithmetic, Boolean logic, and ListOps \citep{nangia_listops_2018}.
In \textbf{modular arithmetic}, the input is a parenthesized expression over digits and the operators \texttt{+} and \texttt{*}, and the label is the value of the full expression modulo 10. 
For example, \texttt{(2 + 3)} maps to \texttt{5}, while \texttt{((2 + 3) * 4)} maps to \texttt{0} modulo 10. 
In \textbf{Boolean evaluation}, the input is a formula over \texttt{1}, \texttt{0}, $\wedge$, $\vee$, and $\sim$ and the label is its truth value. 
For instance, \texttt{(1$\wedge$0)} evaluates to \texttt{0}, while \texttt{($\sim$(1$\vee$0))} also evaluates to \texttt{0}. 
In \textbf{ListOps}, the input is a nested expression over digits with operators such as \texttt{MAX}, \texttt{MIN}, \texttt{SM} (sum modulo), and \texttt{MED}, and the label is the resulting digit. 
For example, \texttt{( MAX 2 7 4 )} evaluates to \texttt{7}, and \texttt{( SM 3 4 8 )} evaluates to \texttt{5} modulo 10, and \texttt{( MAX ( MIN 8 3 5 ) ( SM 4 7 6 ) )} evaluates to \texttt{7}.
Note for ListOps, we limited the number of  fan-in to 3 as our focus was on depth generalization.
See Appendix \ref{app:task} for further details.

\paragraph{Training.}
We train all learnable components of the universal transformer end-to-end from random initialization, including parameters: PE parameters $G_{\text{PE}}$, attention mechanism parameters, and the MLP.
The total number of parameters per task is visualized in Fig. \ref{fig:learn}B; the number of parameters per model is dictated by the size of the vocabulary required for the task and the MLP size, which scales with the finite lookup table).
The $G_{\text{PE}}$ is parameterized by continuous logits that are ternarized to \texttt{{-1,0,1}} with a straight through estimator.
The attention mechanism are likewise initialized from random logits, and binarized with a straight-through estimators in $W_Q$ and $W_K$, while $W_V=I_D$ is fixed (see Fig. \ref{afig:fullmodel}).
The feedforward module is an MLP with a single task-dependent expansive hidden layer (2x embedding dimension for Boolean and arithmetic evaluation and 8x for ListOps) and activation functions (quadratic activation function for Boolean and arithmetic, and ReLU for ListOps), and the output is generated from a gated normalization step to produce a one-hot vector that corresponds to a token in the vocabulary.
Training employed within-epoch curriculum learning: each epoch first trained exclusively on depth-1 expressions, then jointly on depths 1 and 2. 
Models were optimized on a Binary Cross Entropy objective function using Adam (learning rate 0.01) with a two-phase schedule (10\% linear warmup from 30\% to 100\% of base rate, then cosine annealing to 5\%) to address vanishing gradients in shallow expressions. 
Architecturally, models were the same across tasks with the exception of the embedding dimension and MLP size. 
We trained 10 random seeds per model. 
Additional details can be found in Appendix \ref{app:experimental-details}.

\begin{figure}[h]
  \centering
  \includegraphics[scale=1.1]{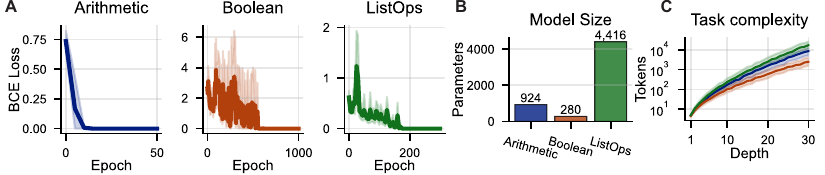}
  \caption{A) Averaged loss trajectories for arithmetic, Boolean, and ListOps tasks (including only seeds that converged).
  (Shaded regions indicate 95\% CI.)
  B) Parameter counts across each of the models, which are dictated by 1) the task vocabulary, and 2) the size of the MLP's hidden dimension. 
  C) The number of input tokens required to express tasks of a given depth. As depth linearly increases, the number of input tokens required to encode a problem rapidly increase (log-linear scale).
  }
  \label{fig:learn}
\end{figure}

\paragraph{Results.} 
Randomly initialized models converged for 10/10 seeds on the modular arithmetic task, and 9/10 seeds for Boolean and ListOps evaluations (Fig. \ref{fig:learn}A). 
During training, optimization proceeded until the model {\em snapped} into an exact solution with zero training loss.
We evaluated trained models on each task using samples with deeper depth than seen during training (Table \ref{table:results}). 
We limited evaluation to problems with up to depth 30 due to memory limitations (depth 30 ListOps problems had upwards of 41k tokens).
The \emph{model halts autonomously} when there remain no parentheses left in the sequence $X^{(t)}$ (i.e., by Eq. \ref{eq:pe_astack}, when $\sum_{i,j} A_{ij}^{(t)}=0$), with the number of iterations matching exactly the depth of the problem.
Critically, our models achieved {\bf 100\% generalization accuracy on all tasks}, irrespective of the problem depth or the number of tokens.
(We tested on 100 problems per depth sampled from a uniform distribution.)
Notably, prior works studying length generalization typically limit evaluations of up to 1k tokens, an order of magnitude less than the problems evaluated here \citep{li_functional_2024,cho_arithmetic_2025,zhou_what_2024,kazemnejad_impact_2023}.
Due to the compactness of the models, we are able to provide complete visualizations of the learned representations in Appendix Fig. \ref{afig:fullmodel}B,C), directly comparing the randomly initialized model, the learned model, and the constructed model defined Section \ref{sec:model}.
We note that while learned models exactly recover the prescribed PE gating parameters defined in the construction, the attention parameters ($W_Q, W_K$) often differed from our formulation due to solution degeneracy.
Nevertheless, all learned models exhibited perfect generalization for arbitrary depth algorithmic problems.

\begin{table}
\caption{We achieved 100\% generalization accuracy on all 3 tasks by training models from random initialization on depth 1 and 2 problems, and evaluating 100 problems at each depth from 3-30.}
\centering
\label{tab:model_summary}
\begin{tabular}{lcccc}
\toprule
Task & Avg Accuracy & Avg Tokens (SD) & Max Tokens & Depths (OOD) \\
\midrule
Arithmetic & \textbf{1.000} & 1937.89 (2725.32) & 18085 & 3-30 \\
Boolean & \textbf{1.000} & 710.85 (914.74) & 7441 & 3-30 \\
ListOps & \textbf{1.000} & 3852.161(5819.27) & 41389 & 3-30 \\
\bottomrule
\end{tabular}
\label{table:results}
\end{table}

\subsection{Comparison to standard architectures}
\label{sec:comparisons}
To test whether the parenthesized structure alone makes these tasks easy, we train a standard looped (weight-shared) transformer \citep{fan_looped_2025} with full softmax attention and learnable $W_Q, W_K, W_V$ under an identical protocol (fully parenthesized inputs, depth-1 and depth-2 curriculum, same iteration count, and matched hyperparameters). 
(We supply the ground-truth depth at test time, as these models do not autonomously halt.) 
We evaluate three positional encodings: learned PEs, RoPE \citep{su_roformer_2022}, and FIRE \citep{li_functional_2024}. 
We report accuracy for seeds that reach $\ge 0.95$ on both training depths. 
We found that these models still collapse towards chance immediately beyond the training depths (Table~\ref{tab:baseline}), while our construction remains exact. 
We further compare to CRvNN \citep{chowdhury2021modeling}, a continuous recursive network previously shown to length generalize well on similar tasks.
However, on ListOps, using our same shallow training protocol and despite roughly $500\times$ more parameters, it degrades from $0.67$ at depth 3 to $0.17$ at depth 10 despite fitting the training depths (Appendix~\ref{app:comparisons}).

\begin{table}
\caption{Comparison to a standard looped transformer. 
We report held-out accuracy at the last training depth ($d=2$) and two out-of-distribution depths ($d=3,10$); chance level in parentheses. 
Accuracy includes seeds that reach $\ge 0.95$ on both training depths. 
On arithmetic and ListOps, only 1/5 seeds fit this criterion per positional encoding within the current training budget. 
On Boolean, 5/5 (learned, RoPE) and 3/5 (FIRE). 
Every trained model still collapses toward chance out of distribution, whereas our construction remains exact.
See Appendix~\ref{app:comparisons} for further details. 
}
\label{tab:baseline}
\centering
\small
\setlength{\tabcolsep}{4pt}
\begin{tabular}{l ccc ccc ccc}
\toprule
 & \multicolumn{3}{c}{Arithmetic (0.10)} & \multicolumn{3}{c}{Boolean (0.50)} & \multicolumn{3}{c}{ListOps (0.10)} \\
\cmidrule(lr){2-4}\cmidrule(lr){5-7}\cmidrule(lr){8-10}
Model & depth-$2$ & depth-$3$ & depth-$10$ & depth-$2$ & depth-$3$ & depth-$10$ & depth-$2$ & depth-$3$ & depth-$10$ \\
\midrule
Ours & \textbf{1.00} & \textbf{1.00} & \textbf{1.00} & \textbf{1.00} & \textbf{1.00} & \textbf{1.00} & \textbf{1.00} & \textbf{1.00} & \textbf{1.00} \\
Looped, learned & 1.00 & 0.10 & 0.05 & 1.00 & 0.55 & 0.48 & 0.98 & 0.37 & 0.07 \\
Looped, RoPE   & 0.73 & 0.18 & 0.00 & 0.98 & 0.50 & 0.12 & 1.00 & 0.43 & 0.08 \\
Looped, FIRE   & 1.00 & 0.23 & 0.03 & 1.00 & 0.61 & 0.44 & 0.97 & 0.38 & 0.12 \\
\bottomrule
\end{tabular}
\end{table}

\section{Discussion and Conclusion}
\label{sec:discussion}


\paragraph{Significance.} 
We give a tiny transformer construction with hand chosen parameters that evaluates arbitrary-depth circuits via reduction.
Empirically, we show that when we randomly initialize our parameterization, models achieve perfect length generalization.
Transformers often fail at symbolic generalization despite strong empirical performance. Recent architectural variants improve length generalization but remain imperfect \citep{cho_arithmetic_2025, zhou_what_2024, fan_looped_2025, kazemnejad_impact_2023, hou_universal_2024}, suggesting a lack of faithful algorithmic implementation.
This empirical gap contrasts with the theoretical results showing that these models are universal in principle \citep{perez_attention_2021,merrill_expressive_2024}.
Yet, such results are largely non-constructive and provide limited insight into how exact algorithmic computation may be mechanistically realized or learned in practice.

We instead give a simple and interpretable construction, repurposing existing architectural primitives used in the literature to implement a structured algorithmic process -- a variant of contextual PE \citep{golovneva_contextual_2024} that delimits positions based on token occurrences, linear attention mechanisms \citep{katharopoulos_transformers_2020}, gated residuals as a variation on hyperconnections \citep{zhu_hyper-connections_2025}, and straight-through estimators throughout to simulate discrete algorithmic representations \citep{bengio_estimating_2013,van_den_oord_neural_2017}.
Experiments further show this behavior is learnable: dense random models acquire ternarized PE gating ($G_{\text{PE}}$), binary attention, and lookup-table MLPs, enabling exact in-context circuit evaluation.

\paragraph{Related work.} 
\emph{Compositional and length generalization:} Generalization to greater compositional complexity remains challenging \citep{ito_quantifying_2025,hupkes_compositionality_2020,lake_generalization_2018,zhou_what_2024,saxton_analysing_2019,kim_cogs_2020,dziri_faith_2023}. 
Prior work improves extrapolation via ad hoc architectural modifications, especially PE in transformers \citep{csordas_devil_2021,kazemnejad_impact_2023,ruoss_randomized_2023,shen_positional_2024,li_functional_2024,mcleish_transformers_2024,cho_arithmetic_2025,jelassi_length_2023}, scratchpads \citep{cho_arithmetic_2025,nye_show_2021}, looped/universal transformers \citep{fan_looped_2025,giannou_looped_2023}, and program compilation approaches \citep{weiss_thinking_2021,shaw_alta_2024}. 
While these improve empirical extrapolation, they lack guarantees of faithful computation at arbitrary depth or are not differentiable. 
Complementary theoretical analyses characterize transformer expressivity limits \citep{strobl_transformers_2024,merrill_little_2025,huang_formal_2025,amiri_lower_2025,yang_masked_2024,deletang_neural_2022}, but rarely provide actual constructions for exact computation.
\emph{Recursive and structured neural architectures:} A complementary line of work composes representations along latent tree structures, including recursive neural networks and their continuous relaxations such as CRvNN \citep{chowdhury2021modeling} (which we compare to directly, see Section~\ref{sec:comparisons}) and beam-tree recursion \citep{ray_chowdhury_efficient_2023}, together with the Neural Data Router \citep{csordas_neural_2022}; \citet{chowdhury2024design} lay out the design space between transformers and recursive nets. 
These models learn composition order through continuous probabilities, achieving strong but imperfect length generalization. 
Our construction instead occupies a maximally constrained point in this space -- deterministic composition order, hard gating, and adaptive halting with a correctness proof -- trading flexibility for exactness.
\emph{Graph neural networks and transformers:} Graph neural networks (GNNs) naturally model circuits via message passing over nodes and edges \citep{scarselli_graph_2009}, and have been studied for algorithmic reasoning \citep{selsam_learning_2019,xu_how_2019,loukas_what_2020}. 
However, standard GNNs require depth proportional to graph diameter and suffer from oversmoothing and oversquashing \citep{alon_bottleneck_2021,gravina_oversquashing_2025}. 
Graph transformers alleviate this with global attention mechanisms \citep{dwivedi_generalization_2021,rampasek_recipe_2023,ying_transformers_2021}, but require $O(n^2)$ cost and often rely on explicit predfined knowledge of the graph structure, rather than learning and computing it on-the-fly as we do here. 
In contrast, our approach uses learned PE gating and structured attention routing to simulate circuit reduction directly. 
\emph{Mechanistic interpretability:} Mechanistic interpretability aims to uncover circuits underlying model behavior \citep{sharkey_open_2025,olah_zoom_2020,olsson_-context_2022,elhage_mathematical_2021,adler_towards_2025}. 
While instructive, these analyses are typically performed post hoc (after learning) and lack guarantees of exact computation.

\paragraph{Limitations.} Our results rely on several assumptions that limit the scope of the current construction while highlighting directions for future work.
First, we assume inputs are fully parenthesized, well-formed expressions, enabling the identification of reducible subexpressions via locally-derived masks. 
Although this setting captures a broad class of algorithmic tasks, it excludes problems where structure must be inferred implicitly.
Fully parenthesized expressions form a bracketed, context-free (Dyck) language whose parse is explicit in the surface string.
The model therefore counts brackets rather than inferring the structure. 
Removing the brackets makes the structure latent -- a structure-induction problem that changes the premise of our theorems -- and natural language, in particular, is generally regarded as beyond context-free \citep{shieber1985evidence}.
Nevertheless, the core components of our construction -- token-dependent positional gating, subexpression masking, and type-constrained attention routing -- provide key ingredients for extending the approach to settings where circuit structure must be learned without parentheses or brackets.
Second, the model evaluates expressions through iterative local reductions and, therefore, does not directly support algorithms with inherently global sequential dependencies such as long addition with carry \citep{dziri_faith_2023,cho_arithmetic_2025}.
Addressing this limitation may require alternative routing mechanisms or an explicit representation of the underlying carry circuit structure, which can still be expressed combinatorially (e.g., Fig. 1 of \citet{dziri_faith_2023}).
Third, our construction also focuses on tasks with a finite fan-in circuit ($\leq 3$). 
Although higher fan-in operations can be reduced to binary compositions, doing so may increase depth and obscure the local semantic structure leveraged by our model.
Designing architectures that support general fan-in while preserving interpretability and generalization remains an important direction for future work.
Finally, our exactness guarantees rely on the discrete machinery of the construction: hardmax attention masking, discrete token states maintained by straight-through estimators, deterministic halting when no parentheses remain, and exact one-hot arithmetic. 
Nevertheless, the learned models remarkably realize these operations with finite-precision relaxations, where straight-through estimators recover the discrete solution.

\paragraph{Conclusion.} 
We present a single-layer, universal transformer that implements circuit computations via circuit reduction, achieving perfect length generalization on Boolean algebra, modular arithmetic, and ListOps evaluation. 
We demonstrate that the construction is learnable, requiring only 924 (arithmetic), 280 (Boolean), and 4416 (ListOps) parameters.
Importantly, this ability is enabled through a novel PE gating mechanism, which identifies reducible subexpressions by tracking problem depth.
This mechanism induces hard attention masking (i.e., scoping) over the input sequence, enforcing structured information flow within the transformer.
It will be important for future work to explore how such hard masking mechanisms can be learned to adapt dynamically, enabling transformers to discover latent structural boundaries (i.e., structure induction) in more complex, unstructured domains.

\bibliography{mybib}


\appendix

\input{appendix}

\end{document}

%% file: appendix.tex
\setcounter{figure}{0}

\renewcommand{\thefigure}{A.\arabic{figure}}

\begin{figure}[h]
  \centering
  \includegraphics[scale=1.0]{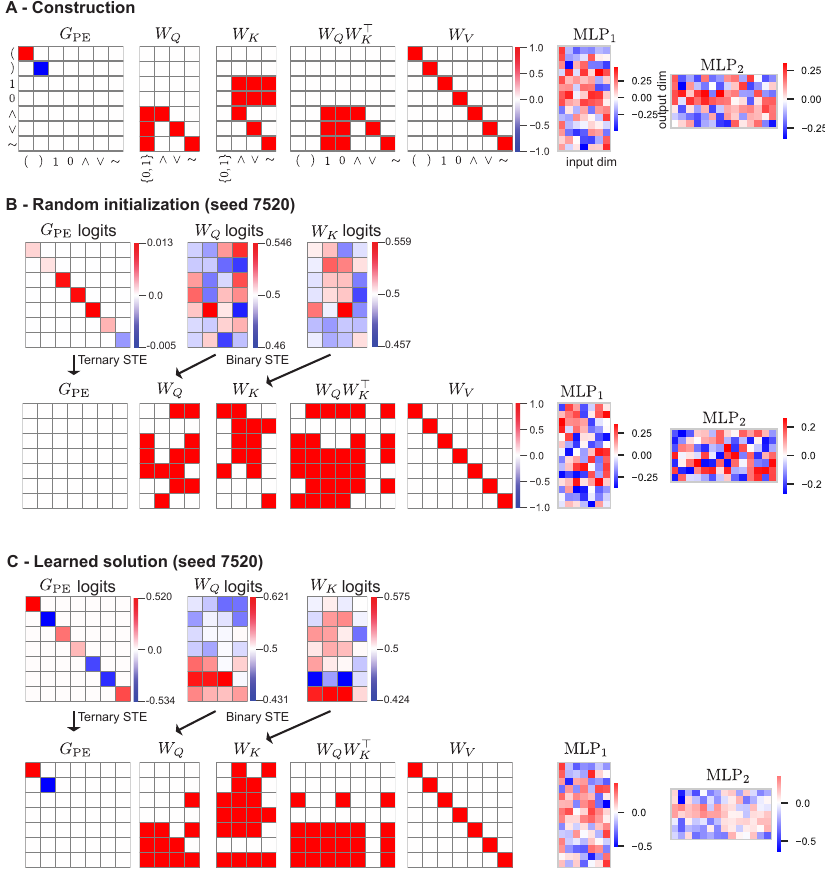}
  \caption{Visualization of all Boolean model parameters for the constructed (formal) solution, a randomly initialized model (before training), and the learned model that exhibits perfect depth generalization.
  Note that for the randomly initialized and learned model, we visualize both the raw logits, as well as the parameter after applying the threshold for binarization/ternarization.
  MLP bias terms are not visualized here.
  A) Our construction of the model that implements circuit reduction, as specified in Theorem \ref{thm:one-step-reduction}-\ref{thm:global-correctness}.
  Here we visualize a learned MLP that is trained on a Boolean look-up table (the 8 depth-1 Boolean trees).
  ($\text{MLP}_1$ refers to the input-to-hidden weights, and $\text{MLP}_2$ refers to the hidden-to-output mapping.)
  B) A randomly initialized model before training for a single seed (7520). 
  Note that logits are typically initialized near or below the straight-through estimator thresholds.
  C) The learned model after training (seed 7520).
  Note that while the model does not learn the exact same parameterization as our constructed solution, this model still exhibits perfect generalization.
  }
  \label{afig:fullmodel}
\end{figure}

\begin{table}[t]
\caption{Model Performance Summary on all 3 tasks fo every depth. Note that depth 1 and 2 are in the training distribution.}
\label{atab:performance_alldepth}
\centering
\begin{tabular}{lcccccc}
\toprule
& \multicolumn{2}{c}{\textbf{Arithmetic}} & \multicolumn{2}{c}{\textbf{Boolean}} & \multicolumn{2}{c}{\textbf{ListOps}} \\
\cmidrule(lr){2-3}\cmidrule(lr){4-5}\cmidrule(lr){6-7}
Depth & Acc. & Tokens & Acc. & Tokens & Acc. & Tokens \\
\midrule
1  & 1.000 & 5.00   & 1.000 & 4.63   & 1.000 & 5.00 \\
2  & 1.000 & 9.64   & 1.000 & 9.00   & 1.000 & 11.28 \\
3  & 1.000 & 16.84  & 1.000 & 14.29  & 1.000 & 19.56 \\
4  & 1.000 & 25.96  & 1.000 & 21.44  & 1.000 & 31.12 \\
5  & 1.000 & 36.92  & 1.000 & 28.90  & 1.000 & 50.52 \\
6  & 1.000 & 54.76  & 1.000 & 37.25  & 1.000 & 70.52 \\
7  & 1.000 & 74.32  & 1.000 & 50.54  & 1.000 & 102.36 \\
8  & 1.000 & 97.64  & 1.000 & 66.31  & 1.000 & 146.80 \\
9  & 1.000 & 127.56 & 1.000 & 82.78  & 1.000 & 178.92 \\
10 & 1.000 & 173.12 & 1.000 & 110.18 & 1.000 & 259.76 \\
11 & 1.000 & 218.88 & 1.000 & 122.81 & 1.000 & 329.64 \\
12 & 1.000 & 288.32 & 1.000 & 168.45 & 1.000 & 418.76 \\
13 & 1.000 & 316.32 & 1.000 & 191.55 & 1.000 & 559.48 \\
14 & 1.000 & 459.72 & 1.000 & 224.85 & 1.000 & 727.52 \\
15 & 1.000 & 570.84 & 1.000 & 295.72 & 1.000 & 961.68 \\
16 & 1.000 & 665.28 & 1.000 & 311.39 & 1.000 & 1086.52 \\
17 & 1.000 & 896.12 & 1.000 & 421.03 & 1.000 & 1493.24 \\
18 & 1.000 & 1025.92& 1.000 & 458.13 & 1.000 & 1722.08 \\
19 & 1.000 & 1254.52& 1.000 & 535.26 & 1.000 & 2239.52 \\
20 & 1.000 & 1572.56& 1.000 & 659.82 & 1.000 & 2746.20 \\
21 & 1.000 & 1892.24& 1.000 & 749.34 & 1.000 & 3584.56 \\
22 & 1.000 & 2130.96& 1.000 & 919.54 & 1.000 & 4100.28 \\
23 & 1.000 & 2654.28& 1.000 & 1057.65& 1.000 & 5199.72 \\
24 & 1.000 & 3022.92& 1.000 & 1217.01& 1.000 & 6126.32 \\
25 & 1.000 & 4068.36& 1.000 & 1496.06& 1.000 & 7593.32 \\
26 & 1.000 & 4560.12& 1.000 & 1553.46& 1.000 & 8747.92 \\
27 & 1.000 & 5484.48& 1.000 & 1859.65& 1.000 & 11437.36 \\
28 & 1.000 & 6442.20& 1.000 & 1954.27& 1.000 & 14215.60 \\
29 & 1.000 & 7620.96& 1.000 & 2548.09& 1.000 & 15335.56 \\
30 & 1.000 & 8509.00& 1.000 & 2748.17& 1.000 & 18375.68 \\
\bottomrule
\end{tabular}
\end{table}

\section{Model construction details}

\subsection{Computational complexity of model construction}
\label{app:complexity}
Theorem~\ref{thm:one-step-reduction} characterizes the cost of a single universal-transformer iteration.
For an input of length $n$, reduction mask construction requires parenthesis counting and thresholding over the sequence, which is linear in $n$ (Lemma \ref{lem:reduction-mask}). 
Subexpression decomposition is likewise linear in $n$, since it consists of a discrete difference and cumulative sum over the binary mask (Lemma \ref{lem:cluster-correctness}). 
The dominant cost is the typed routing step, whose linear-attention implementation scales as $O(kn)$, where $k$ is the number of parallel reducible subexpressions. 
For the parenthesized expressions considered here -- where $k \ll n$ -- the effective cost is near-linear in $n$. 
Thus, the dominant cost per iteration is $O(n)$, aside from storing the fixed routing parameters (e.g., attention head dimension $d_h$).

By Theorem~\ref{thm:global-correctness}, the number of transformer iterations required to evaluate an expression $x$ requires exactly $d(x)$ iterations, where $d(x)$ is the depth of the original expression.
Because each additional level of nesting introduces at least one pair of parentheses and one operator, $d(x) \le n/3$.
This implies that the total cost $O(n \cdot d(x))$ is always sub-quadratic; the worst case $O(n^2)$ arises only for maximally unbalanced expressions with $d(x) = \Theta(n)$.
For balanced expressions $d(x) = O(\log n)$, giving an effective total cost of $O(n \log n)$.
In this sense, the construction has lower sequential depth than standard sequential token-by-token evaluation procedures, such as stack-based parsing or left-to-right symbolic evaluation, which typically require $\Theta(n)$ sequential steps on an input of length $n$ (e.g., see \citet{deletang_neural_2022}).

\subsection{Additional proofs for model construction}
\label{app:proofs}
\mlplemma*

\begin{proof}
\label{app:proof_mlp}
The set of valid reducible subexpressions over fan-in-2 Boolean trees is finite; enumeration yields exactly 8 distinct cases ($\wedge,\vee,\sim$ operators applied to \texttt{0}, \texttt{1} operands.
We list these explicitly as input-output pairs $(\mathbf{b}_i, \mathbf{y}_i )_{i=1}^8$, where $\mathbf{b}_i \in \mathbb{Z}^D$ is the bag-of-words representation of the $i$-th subexpression and $\mathbf{y}_i \in \{0,1\}^D$ is its evaluated output embedding.
Any input $\mathbf{b}$ not matching a valid reducible expression maps to $\mathbf{0}^D$.

Since the domain $\{\mathbf{b}_1,...,\mathbf{b}_8\}$ is finite, $f$ is simply a lookup table on 8 entries.
By the universal approximation theorem \citet{hornik_multilayer_1989}, any function on a finite domain is exactly realizable by a single-hidden-layer MLP with sufficiently many units and a suitable activation function.
Hence, such an $f$ exists.
\end{proof}

\gatedlemma*
\begin{proof}
By Lemma \ref{lem:reduction-mask}, $P^{(t)}$ denotes the mask of reducible subexpressions in $X^{(t)}$, and let $Y^{(t)} \in \{0,1\}^{n \times V}$ denote the sequence output vector after the feedforward layer (Lemma \ref{lem:mlp}). 
For each position $i$,
$$
X^{(t+1)}_i = M^{(t)}_i X^{(t)}_i + P^{(t)}_i Y^{(t)}_i.
$$

Since $M^{(t)}_i = 1 - P^{(t)}_i$ and $P^{(t)}_i \in \{0,1\}$, exactly one of $M^{(t)}_i$ or $P^{(t)}_i$ equals 1. 
Thus, if $P^{(t)}_i = 0$, then $X^{(t+1)}_i = X^{(t)}_i$, and if $P^{(t)}_i = 1$, then $X^{(t+1)}_i = Y^{(t)}_i$.
Thus, positions outside the reduction mask are preserved and those inside are updated.
\end{proof}

\thmone*
\begin{proof}
We show that one iteration evaluates all deepest reducible subexpressions while preserving all other tokens.
By Lemma \ref{lem:reduction-mask}, the reduction mask identifies exactly the deepest reducible subexpressions. 
By Lemma \ref{lem:cluster-correctness}, parallel subexpression decomposition isolates independent subexpressions. 
By Lemma \ref{lem:attn}, the attention mechanism takes each independent subexpression, and sums together local operand and operator embeddings at the operator position creating a bag-of-words vector embedding.
Tokens across independent subexpressions do not interfere with subexpressions, because they are disjoint subsequences (Lemma \ref{lem:cluster-correctness}).
By Lemma \ref{lem:mlp}, the bag-of-words embedding vector at the operator's token position is evaluated to its semantic value; all other tokens within that subexpression are set to the 0 vector. 
By Lemma \ref{lem:residual-preservation} the residual update preserves all remaining non-reduced positions. 
Since the MLP reduces all reducible subexpressions (Lemma \ref{lem:mlp}) and the gated residual connection preserves all other tokens (Lemma \ref{lem:residual-preservation}), this implies $X^{(t+1)} \equiv X^{(t)}$.
Since only the deepest subexpressions are reduced, all nodes at maximal depth are eliminated and no deeper nodes remain (parentheses of the deepest subexpressions are replaced by 0 vectors by Lemma \ref{lem:mlp}).
This yields $\delta(X^{(t+1)}) = \delta(X^{(t)}) - 1$. 
Hence one iteration performs a parallel local reduction step, preserving the global evaluation and decreasing active depth by one.
\end{proof}

\thmtwo*
\begin{proof}
We apply Theorem \ref{thm:one-step-reduction} inductively on $t$.
At $t=0$, $X^{(0)}$ corresponds to $x$, so $\llbracket X^{(0)} \rrbracket \equiv \llbracket X^{(0)} \rrbracket$ trivially with $\delta(X^{(0)})=d(x)$.

Assume for $t<d(x)$ that
$$
X^{(t)} \equiv X^{(0)} \quad \text{and} \quad \delta(X^{(t)}) = d(x) - t.
$$

Since $\delta(X^{(t)})>0$, we can apply Theorem \ref{thm:one-step-reduction}, yielding

$$
X^{(t+1)} \equiv X^{(t)} \quad \text{and} \quad \delta(X^{(t+1)}) = \delta(X^{(t)})-1
$$

By induction
$$X^{(t+1)} \equiv X^{(0)} \quad \text{and} \quad \delta(X^{(t+1)}) = d(x)-(t+1)$$

Thus, $\forall t \le d(x)$
$$X^{(t)} \equiv X^{(0)} \quad \text{and} \quad \delta(X^{(t)}) = d(x) - t$$

When $t=d(x)$, we have $\delta(X^{(d(x))})=0$, so no reducible subexpression remains, and $X^{(d(x))}$ is a Boolean atom (i.e., \texttt{\{0,1\}}).
Since semantic equivalence is preserved throughout, we have
$$
X^{(d(x))} \equiv \llbracket x \rrbracket$$
\end{proof}

\section{Experimental Details}
\label{app:experimental-details}

\subsection{Task details}
\label{app:task}

\paragraph{Arithmetic.}
Each arithmetic example is a fully parenthesized expression built recursively from digits \(\{0,\dots,9\}\) and the binary operators \texttt{+} and \texttt{*}. 
The output is the numerical value of the expression modulo \(10\). 
Depth-0 instances are single digits. 
Depth-1 instances contain a single operation, e.g., \texttt{(7 * 8)} \(\mapsto 6\). 
Larger-depth instances are formed by nesting subexpressions, e.g., \texttt{((1 + 2) * (3 + 4))} \(\mapsto 1\). 
To illustrate the complexity of deeper problems, a depth-3 example is \texttt{(((1 + 2) * (3 + 4)) + (5 * (6 + 7)))} \(\mapsto 6\).
On average, the number of tokens exponentially increases as the problem's depth increases.

\paragraph{Boolean.}
Each Boolean example is a recursively defined expression over leaves \texttt{TRUE} and \texttt{FALSE}, unary \texttt{NOT}, and binary \texttt{AND}/\texttt{OR}. 
The output is the truth value of the full formula. 
A depth-3 formula \texttt{(((NOT FALSE) AND TRUE) OR (FALSE AND (NOT TRUE)))} \(\mapsto\) \texttt{TRUE}. 

\paragraph{ListOps.}
Each ListOps example is a nested expression whose internal nodes are operators and whose leaves are digits \(0\) through \(9\) (see \citet{nangia_listops_2018}; MIT License). 
The operators are \texttt{MAX} (maximum), \texttt{MIN} (minimum), \texttt{SM} (sum modulo \(10\)), and \texttt{MED} (median). 
For example, \texttt{( MIN 8 3 5 )} \(\mapsto 3\), \texttt{( MAX 1 9 4 )} \(\mapsto 9\), and \texttt{( MED 8 2 5 )} \(\mapsto 5\). 
A nested example is \texttt{( MAX ( MIN 8 3 5 ) ( SM 4 7 6 ) )}, which evaluates to \texttt{7}, since the inner expressions yield \(3\) and \(7\), respectively. 
Note for our experiments, we limited the fan-in (or number of arguments per operator) to max 3. 
This meant randomly sampled operators could have either 2 or 3 arguments.
By default, operator \texttt{MED} chose the floor operand for fan-in 2 (or for sequences with repeating operands).

Note that Lemma $\ref{lem:cluster-correctness}$ assumes that independently evaluable subexpressions are non-adjacent, which holds for in-fix grammars such as Boolean formulas and modular arithmetic considered here.
However, for prefix grammars, such as ListOps, subexpressions are adjacent (i.e., there is no separation between the closing parenthesis of one expression and the opening parenthesis of the next).
In principle, this solution can be addressed by tokenizing whitespace -- for example, even assigning a $\mathbf{0}$ vector to whitespace tokens. 
However, this approach doubles the total number of tokens.
Thus, in practice, we instead tokenize closing parentheses using two tokens: a one-hot encoded vector followed by an adjacent $\mathbf{0}$ vector.
This increases the number of input tokens in proportion to the number of closed parentheses.

\subsection{Model Training}

\paragraph{Overview.}
All models were trained end-to-end from random initialization using iterative computation, where the number of iterations was specified by problem depth.
Training therefore directly optimizes the full multi-step reduction procedure rather than a surrogate objective on intermediate states.

\paragraph{Domains.} We conducted experiments across three symbolic evaluation tasks:
\begin{itemize}
    \item \textbf{Boolean logic}: Expressions over the vocabulary $\Sigma_{\mathrm{bool}} = \{\texttt{(}, \texttt{)}, \texttt{1}, \texttt{0}, \wedge, \vee, \sim\}$ with operators $\{\wedge, \vee, \sim\}$.
    \item \textbf{Modular arithmetic}: Expressions over $\Sigma_{\mathrm{arith}} = \{\texttt{(}, \texttt{)}, 0, 1, \ldots, 9, \texttt{+}, \texttt{*}\}$ evaluated modulo 10.
    \item \textbf{ListOps}: Hierarchical list operations following the benchmark of \citet{nangia_listops_2018}, which tests compositional reasoning over nested list structures.
\end{itemize}

\paragraph{Input representation and initialization.}
Tokens are represented as fixed one-hot vectors.
In the experiments reported, the model dimension equals the vocabulary size for the corresponding task domain.
The model operates directly on the one-hot representations.
However, the model can similarly operate on dense input embeddings via an orthogonal embedding map.
All learnable parameters are initialized randomly with small magnitude so that training begins near the discrete decision boundaries used by the straight-through estimators.

\paragraph{Positional encoding gate.}
The PE mechanism is parameterized by a learnable vector of continuous logits, one per embedding dimension. 
During training, these logits are converted to ternary values in $\{-1,0,1\}$ using fixed thresholds of $-0.5$ and $0.5$.
In the forward pass, dimensions with logits below $-0.5$ are mapped to $-1$, while dimensions above the upper threshold are mapped to $+1$, and the remaining dimensions are mapped to $0$.
The resulting diagonal gate matrix is then used to count open and closed parentheses to identify the most deeply nested subexpressions.
Importantly, when there are no remaining parentheses the model halts.
Although the forward PE is discrete, gradients are passed to the underlying logits as if the quantization step were the identity, enabling standard gradient-based optimization.

\paragraph{Attention mechanism.}
The query and key matrices $W_Q$ and $W_K$ are learned as binary matrices using straight-through estimators. 
Each matrix is parameterized by a real-valued logit matrix. 
During training, the logits are passed through a sigmoid and thresholded at $0.5$ to obtain a binary routing pattern, while gradients flow through the continuous sigmoid.
This setup encourages sparse routing patterns that approximate the constructive circuit design.
We fix the value matrix $W_V$ to the identity, so attention only redistributes existing token information rather than learning a separate value transformation.

\paragraph{Feedforward layer.}
After attention, the resulting local count-like representation is processed by a standard two-layer MLP.
This is the only dense component of the model. 
For Boolean and arithmetic tasks, we use a hidden dimension equal to $2D$ and a quadratic activation. 
For ListOps, we use a wider hidden layer of size $8D$ together with a ReLU activation.
(Note here that the MLP embedding dimensions were determined by simple experimentation, to identify when the model would converge.)
The role of the MLP is to map routed local representations to the semantic value of a subexpression, i.e., a finite-lookup table for each task domain.
After the MLP, the model applies a gated normalization step to generate a one-hot vector that corresponds to the look-up table.
In the forward pass, activations are thresholded to $\{0,1\}$; in the backward pass, gradients are propagated through the corresponding continuous activations (another straight-through estimator).
This encourages the internal token states to remain discrete symbolic representations across iterations, which is key for a single transformer iteration to implement an exact, depth-1 circuit reduction.

\paragraph{Straight-through vs.\ Gumbel-Softmax.} We also experimented with a Gumbel-Softmax relaxation in place of the straight-through estimators for the discrete components. 
Training was considerably less stable; models did not reliably \emph{snap} to the exact solution, in contrast to the straight-through estimators used throughout. 
A dedicated study of discrete-gradient estimators for this class of constructions would be an interesting direction for future work.

\paragraph{Gated residual update.}
We include a gated residual update, whereby non-reducible portions of the input string (i.e., tokens outside the innermost expressions) are residually added to the output of the feedforward layer.
Success of this update is determined by the success of learning the correct PE gating matrix $G_{\text{PE}}$ (see \ref{eq:pe_p}). 

\paragraph{Training Data and Optimization.}
For each domain, we generated fully enumerated training sets containing all possible valid expressions at depths 1 and 2. 
(However, for ListOps, because enumerating all possible depth 2 circuits was computationally inefficient due to the large combinatorial space (ListOps problems had up to 3 arguments), we opted to uniformly randomly sample new sets of depth 2 circuits after each epoch.
The size of the depth-2 dataset we sampled for each epoch was 2x the size of the enumerated depth-1 task set, which was 4400 samples for all depth-1 problems, and thus we randomly sampled 8800 depth-2 ListOps problems.)
The enumeration of depth-1 circuits ensured complete coverage of shallow compositional patterns while maintaining tractable dataset sizes. 
Training sets were balanced across output classes (e.g., TRUE/FALSE for Boolean, 0-9 for arithmetic) to prevent class imbalance during optimization.
Models were optimized with Adam using a Binary Cross Entropy loss.
Rather than calculating the loss from a single classification token, \emph{the loss was calculated across all tokens in the sequence}.
In particular, the target sequence was generated following the circuit reduction procedure we described in Theorem \ref{thm:one-step-reduction} -- the operator token at the highest (shallowest) token position encoded the final target class, while the embeddings of all other tokens in the sequence were trained to be the 0 vector.
While this required us to keep track of which token position encoded the correct target class in the training data, at test time, the max pooling of the output's embedding dimensions of size $n \times D$ across all $n$ tokens exactly encoded the correct one-hot token embedding.
This was because non-output token positions were optimized to be the 0 vector.
Across experiments, we used a learning rate of $0.01$, batch size of $500$, no dropout, and gradient clipping with a maximum norm of 1.0.
Note that for depth 1 circuits, the number of total samples was typically fewer than $500$. 
In those cases, batch size would reduce to the total number of enumerated depth 1 samples.
We trained models for a fixed number of epochs, which were dictated by the number of total samples per circuit depth: 50 epochs for modular arithmetic, 1000 for boolean expressions, and 300 for ListOps.

\paragraph{Curriculum learning.} 
We employed a progressive curriculum learning strategy. 
Within each epoch, the model first trained exclusively on depth-1 expressions, then on both depth-1 and depth-2 expressions together. 
This curriculum approach helped the model first learn the fundamental depth-1 reduction operations before tackling more complex nested structures (which were aimed at learning the PE gating parameters).

\paragraph{Learning rate schedule.} To address vanishing gradient issues in shallow expressions (depths 1-2), we employed a two-phase learning rate schedule:
\begin{enumerate}[leftmargin=2em]
    \item \textbf{Warmup phase} (10\% of total steps): Linear warmup from 30\% to 100\% of the base learning rate. 
    \item \textbf{Cosine annealing phase} (remaining 90\% of steps): Cosine decay from the peak learning rate to 5\% of the base rate ($\eta_{\min} = 0.0005$). 
\end{enumerate}


\paragraph{Computational environment.} 
All experiments could be efficiently implemented on a CPU (Macbook).
This ensured numerical stability and reproducibility. 
(Anecdotally, we observed that CPU training produced more consistent convergence behavior across random seeds, likely due to deterministic floating-point operations.)
Each seed took no longer than 2 hours to train per seed (Boolean models converged within 2 hours for all 10 seeds).
However, once learned, we evaluated models (particularly with large context inputs) on an H100 GPU, achieving perfect performance.

\paragraph{Random seeds.} We trained 10 independent model instances per domain using the following random seeds: 8739, 7520, 1936, 7439, 6210, 5650, 5395, 3140, 1243, 5353. 
These seeds controlled initialization of model parameters, data shuffling, and any stochastic operations during training. 
Seed 1936 failed to converge for Boolean evaluation, and seed 7439 failed to converge for ListOps.

\subsection{Evaluation Protocol}

\paragraph{Generalization testing.} After training on depths 1-2, models were evaluated on held-out test sets at depths 3-30.
These deeper expressions were never seen during training, allowing us to measure depth/length generalization.

\subsection{Code \& Implementation Details}

The code required to run experiments is provided in the supplementary zip file, along with a README and several python notebooks.
The trained model required to produce visualizations in Fig. \ref{afig:fullmodel} is also included, along with the corresponding notebook \path{viz_boolean_params.ipynb}.

For reproducibility, the exact command used to train models was:

\begin{verbatim}
python run_transformer_enumerated.py \
    --model_type {boolean,arithmetic,listops} \
    --seeds 8739 7520 1936 7439 6210 5650 5395 3140 1243 5353 \
    --min_train_depth 1 \
    --max_train_depth 2 \
    --n_epochs {1000,50,300} \
    --learning_rate 0.01 \
    --wdecay 0.0 \
    --device cpu \
    --learnable_pe \
    --learnable_attn \
    --learnable_mlp \
    --optimizer adamw \
    --train_batch_size 500 \
    --balanced \
    --loss bce
\end{verbatim}

where \texttt{model\_type} was set to \texttt{boolean}, \texttt{arithmetic}, or \texttt{listops} for the respective experiments.

\section{Comparison to standard architectures: details}
\label{app:comparisons}

We expand the comparison summarized in Section~\ref{sec:comparisons} and Table~\ref{tab:baseline} in this section. 
The looped-transformer baseline is a standard weight-shared transformer with full softmax attention and learnable $W_Q, W_K, W_V$, trained on the same fully parenthesized inputs, the same depth-1 and depth-2 curriculum, and the same per-task epoch budget as our own models (50 epochs for arithmetic, 1000 for Boolean, 300 for ListOps) \citep{fan_looped_2025}. 
Because the baselines do not autonomously halt, we supply the ground-truth depth at test time and run exactly that many iterations. 
We evaluated five seeds per positional encoding (8739, 7520, 1936, 7439, 6210) and selected the best checkpoint by training accuracy.

We report accuracy conditioned on the seeds that reach $\ge 0.95$ on both training depths. 
For arithmetic and ListOps, only one seed per positional encoding meets this criterion; for Boolean, all five seeds fit under learned and RoPE encodings, and three of five under FIRE. 
Visual inspection of the training trajectories showed that the non-fitting seeds are not fully converged after the fixed number of epochs. 
We therefore do not claim that the baseline architecture \emph{cannot} fit the training depths; rather, among seeds that do fit within a matched budget, every one collapses toward chance out of distribution (Table~\ref{atab:baseline_full}).

\begin{table}
\caption{Full generalization accuracy of the looped-transformer baseline across depths, per task and positional encoding, including only fitted seeds. Chance level is 0.10 for arithmetic and ListOps, and 0.50 for Boolean. Our construction achieves 1.0 at every depth on all three tasks (Table~\ref{table:results}).}
\label{atab:baseline_full}
\centering
\small
\begin{tabular}{lccccc}
\toprule
Positional encoding & $d\!=\!1$ & $d\!=\!2$ & $d\!=\!3$ & $d\!=\!5$ & $d\!=\!10$ \\
\midrule
\multicolumn{6}{l}{\emph{Arithmetic (chance 0.10)}} \\
Learned & 1.00 & 1.00 & 0.10 & 0.05 & 0.05 \\
RoPE    & 1.00 & 0.73 & 0.18 & 0.18 & 0.00 \\
FIRE    & 1.00 & 1.00 & 0.23 & 0.13 & 0.03 \\
\midrule
\multicolumn{6}{l}{\emph{Boolean (chance 0.50)}} \\
Learned & 1.00 & 1.00 & 0.55 & 0.50 & 0.48 \\
RoPE    & 1.00 & 0.98 & 0.50 & 0.27 & 0.12 \\
FIRE    & 1.00 & 1.00 & 0.61 & 0.48 & 0.44 \\
\midrule
\multicolumn{6}{l}{\emph{ListOps (chance 0.10)}} \\
Learned & 1.00 & 0.98 & 0.37 & 0.12 & 0.07 \\
RoPE    & 1.00 & 1.00 & 0.43 & 0.15 & 0.08 \\
FIRE    & 1.00 & 0.97 & 0.38 & 0.20 & 0.12 \\
\bottomrule
\end{tabular}
\end{table}

\paragraph{Continuous recursive networks (CRvNN).}
We additionally compared to CRvNN \citep{chowdhury2021modeling}, a continuous relaxation of recursive neural networks, on the ListOps task under the same fan-in and depth 1 and 2 training protocol. 
We used the original architecture unmodified (521,611 parameters, roughly $500\times$ ours) and trained five seeds, selecting the best checkpoint by validation accuracy. 
CRvNN fitted the training depths but degraded out of distribution (Table~\ref{atab:crvnn}), in contrast to the exact generalization of our construction. 
This mirrors the looped-transformer result: expressive, higher-capacity models fitted the shallow training depths yet failed to generalize exactly, whereas our constrained construction remains exact at every depth.

\begin{table}
\caption{CRvNN \citep{chowdhury2021modeling} versus our construction on ListOps (10-way classification, chance 0.10), reported as mean (standard deviation) over five seeds. 
CRvNN has 521,611 parameters. Our construction attained 1.000 at every depth.}
\label{atab:crvnn}
\centering
\small
\setlength{\tabcolsep}{4pt}
\begin{tabular}{lccccccc}
\toprule
Model & $d\!=\!1$ & $d\!=\!2$ & $d\!=\!3$ & $d\!=\!4$ & $d\!=\!5$ & $d\!=\!8$ & $d\!=\!10$ \\
\midrule
Ours & \textbf{1.000} & \textbf{1.000} & \textbf{1.000} & \textbf{1.000} & \textbf{1.000} & \textbf{1.000} & \textbf{1.000} \\
CRvNN & 1.000 (.000) & 0.961 (.060) & 0.674 (.186) & 0.431 (.110) & 0.302 (.065) & 0.177 (.071) & 0.169 (.056) \\
\bottomrule
\end{tabular}
\end{table}